\documentclass{article}
\usepackage[utopia]{mathdesign}
\usepackage{graphicx}
\usepackage{amsmath}
\usepackage{array}
\usepackage{amsthm}
\usepackage[colon,sort&compress]{natbib}
\usepackage{enumerate}
\usepackage{bm}

\newtheorem{theorem}{Theorem}[section]
\newtheorem{lemma}[theorem]{Lemma}

\newtheorem{proposition}[theorem]{Proposition}

\theoremstyle{definition}

\newtheorem{remark}[theorem]{Remark}

\renewcommand{\tilde}{\widetilde}
\renewcommand{\hat}{\widehat}
\renewcommand{\Pr}{\mathbb{P}}
\renewcommand{\Re}{\mathbb{R}}

\begin{document}

\title{Universally Consistent Latent Position Estimation and Vertex Classification for Random Dot Product Graphs}
\author{Daniel L. Sussman, Minh Tang, Carey E. Priebe\\
Johns Hopkins University, Applied Math and Statistics Department}
%\author[jhu]{Daniel L. Sussman}
%\ead{dsussma3@jhu.edu}
%\author[jhu]{Minh Tang}
%\ead{mtang10@jhu.edu}
%\author[jhu]{Carey E. Priebe \corref{cor1}}
%\ead{cep@jhu.edu}
%\address[jhu]{Johns Hopkins University, Applied Math and Statistics Department, Whitehead Hall, Baltimore, MD, USA 21218-2682}

\maketitle

\begin{abstract}
In this work we show that, using the eigen-decomposition of the adjacency matrix, we can consistently estimate latent positions for random dot product graphs provided the latent positions are i.i.d.\ from some distribution. If class labels are observed for a number of vertices tending to infinity, then we show that the remaining vertices can be classified with error converging to Bayes optimal using the $k$-nearest-neighbors classification rule. We evaluate the proposed methods on simulated data and a graph derived from Wikipedia.
\end{abstract}

%\begin{keyword}
%random graph, $k$-nearest-neighbor, latent space model, universal consistency
%\end{keyword}

\section{Introduction}

The classical statistical pattern recognition setting
involves $$(X,Y),(X_1,Y_1),\dotsc,(X_n,Y_n) \stackrel{i.i.d.}{\sim} F_{X,Y},$$
where the $X_i:\Omega\mapsto \mathbb{R}^d$ are observed feature vectors
and the $Y_i:\Omega\mapsto\{0,1\}$ are observed class labels for some probability space $\Omega$. We define $\mathcal{D}=\{(X_i,Y_i)\}$ as the training set. 
The goal is to learn a classifier $h(\cdot;\mathcal{D}): \mathbb{R}^d \to \{0,1\}$
such that the probability of error $\Pr[h(X;\mathcal{D})\neq Y|\mathcal{D}]$
approaches Bayes optimal as $n\to \infty$ for all distributions $F_{X,Y}$ -- universal consistency \citep{devroye1996probabilistic}.
Here we consider the case wherein the feature vectors $X,X_1,\dotsc,X_n$ are unobserved,
and we observe instead a latent position graph $G(X,X_1,\dotsc,X_n)$ on $n+1$ vertices.
We show that a universally consistent classification rule (specifically, $k$-nearest neighbors)
remains universally consistent for this extension of the pattern recognition set up to latent position graph models.

Latent space models for random graphs \citep{Hoff2002} offer a framework in which a graph structure can be parametrized by latent vectors associated with each vertex.
Then, the complexities of the graph structure can be characterized usings  well-known techniques for vector spaces. 
One approach, which we adopt here, is that given a latent space model for a graph, we first estimate the latent positions and then use the estimated latent positions to perform subsequent analysis.
When the latent vectors determine the distribution of the random graph, accurate estimates of the latent positions will often lead to accurate subsequent inference.

In particular, this paper considers the random dot product graph model introduced in \cite{nickel2006random} and \cite{young2007random}.
This model supposes that each vertex is associated with a latent vector in $\Re^d$.
The probability that two vertices are adjacent is then given by the dot product of their respective latent vectors.
We investigate the use of an eigen-decomposition of the observed adjacency matrix to estimate the latent vectors.
The motivation for this estimator is that, had we observed the expected adjacency matrix (the matrix of adjacency probabilities), then this eigen-decomposition would return the original latent vectors (up to an orthogonal transformation).

%We show that, provided the latent vectors are i.i.d.\ from any suitable distribution $F$, we can accurately recover the latent positions. 
Provided the latent vectors are i.i.d.\ from any distribution $F$ on a suitable space $\mathcal{X}$, we show that we can accurately recover the latent positions.
Because the graph model is invariant to orthogonal transformations of the latent vectors, note that the distribution $F$ is identifiable only up to orthogonal transformations.
Consequently, our results show only that we estimate latent positions which can then be orthogonally transformed to be close to the true latent vectors. As many subsequent inference tasks are invariant to orthogonal transformations, it is not necessary to achieve a rotationally accurate estimate of the original latent vectors. 

For this paper, we investigate the inference task of vertex classification. This supervised or semi-supervised problem supposes that we have observed class labels for some subset of vertices and that we wish to classify the remaining vertices.
To do this, we train a $k$-nearest-neighbor classifier on estimated latent vectors with observed class labels, which we then use to classify vertices with un-observed class labels. 
Our result states that this classifier is universally consistent, meaning that regardless of the distribution for the latent vectors, the error for our classifier trained on the estimated vectors converges to Bayes optimal for that distribution. 

%\begin{algorithm}
%\begin{algorithmic}
%
%\end{algorithmic}
%\end{algoritm}

The theorems as stated can be generalized in various ways without much additional work.
For ease of notation and presentation, we chose to provide an illustrative example for the kind of results that can be achieved for the specific random dot product model. 
In the discussion we point out various ways that this can be generalized. 

The remainder of the paper is structured as follows. 
Section~\ref{sec:relwork} discusses previous work related to the latent space approach and spectral properties of random graphs.
In section~\ref{sec:frame}, we introduce the basic framework for random dot product graphs and our proposed latent position estimator.
In section~\ref{sec:rdpg}, we argue that the estimator is consistent, and in section~\ref{sec:knn} we show that the $k$-nearest-neighbors algorithm yields consistent vertex classification.
In section~\ref{sec:disc} we consider some immediate ways the results presented herein can be extended and discuss some possible implications. 
Finally, section~\ref{sec:emp} provides illustrative examples of applications of this work through simulations and a graph derived from Wikipedia articles and hyper-links.

\section{Related Work} \label{sec:relwork}
The latent space approach is introduced in \cite{Hoff2002}. Generally, one posits that the adjacency of two vertices is determined by a Bernoulli trial with parameter depending only on the latent positions associated with each vertex, and edges are independent conditioned on the latent positions of the vertices.

If we suppose that the latent positions are i.i.d.\ from some distribution, then  the latent space approach is closely related to the theory of exchangeable random graphs \citep{Bickel2009,kallenberg2005probabilistic,aldous1981representations}. 
For exchangeable graphs, we have a (measurable) link function $g:[0,1]^2\mapsto[0,1]$ and each vertex is associated with a latent i.i.d.\ uniform $[0,1]$ random variable denoted $X_i$. 
Conditioned on the $\{X_i\}$, the adjacency of vertices $i$ and $j$ is determined by a Bernoulli trial with parameter $g(X_i,X_j)$.
For a treatment of exchangeable graphs and estimation using the method of moments, see \cite{bickel2011method}.

The latent space approach replaces the latent uniform random variables with random variables in some $\mathcal{X}\subset\Re^d$, and the link function $g$ has domain $\mathcal{X}^2$. These random graphs still have exchangeable vertices and so could be represented in the i.i.d.\ uniform framework. On the other hand, $d$-dimensional latent vectors allow for additional structure and advances interpretation of the latent positions.

In fact, the following result provides a characterization of {\em finite-dimensional} exchangeable graphs as random dot product graphs. First,  we say $g$ is  rank $d<\infty$ and positive semi-definite if $g$ can be written as $g(x,y)=\sum_{i=1}^d \psi_i(x)\psi_i(y)$ for some linearly independent functions $\psi_j:[0,1]\mapsto [-1,1]$. Using this definition and the inverse probability transform, one can easily show the following.
\begin{proposition}
 An exchangeable random graph has rank $d<\infty$ and positive semi-definite link function if and only if the random graph is distributed according to a random dot product graph with i.i.d. latent vectors in $\Re^d$.
\end{proposition}
\noindent Put another way, random dot products graphs are exactly the finite-dimensional exchangeable random graphs, and hence, they represent a key area for exploration when studying exchangeable random graphs.

An important example of a latent space model is the stochastic blockmodel \citep{Holland1983}, where each latent vector can take one of only $b$ distinct values. 
The latent positions can be taken to be $\mathcal{X}=[b]=\{1,\dotsc,b\}$ for some positive integer $b$, the number of blocks. Two vertices with the same latent position are said to be members of the same block, and block membership of each vertex determines the probabilities of adjacency. Vertices in the same block are said to be stochastically equivalent. This model has been studied extensively, with many efforts focused on unsupervised estimation of vertex block membership \citep{Snijders1997Estimation,Bickel2009,Choi2010}. Note that \cite{STFP-2011} discusses the relationship between stochastic blockmodels and random dot product graphs.
The value of the stochastic blockmodel is its strong notions of communities and parsimonious structure; however the assumption of stochastic equivalence may be too strong for many scenarios.

Many latent space approaches seek to generalize the stochastic blockmodel to allow for variation within blocks. For example, 
the mixed membership model of \cite{Airoldi2008} posits that a vertex could have partial membership in multiple blocks. In \cite{Handcock2007Modelbased}, latent vectors are presumed to be drawn from a mixture of multivariate normal distributions with the link function depending on the distance between the latent vectors. They use Bayesian techniques to estimate the latent vectors.

Our work relies on techniques developed in \cite{rohe2011spectral} and \cite{STFP-2011} to estimate latent vectors. In particular, \cite{rohe2011spectral} prove that the eigenvectors of the normalized Laplacian can be orthogonally transformed to closely approximate the eigenvectors of the population Laplacian. Their results do not use a specific model but rather rely on assumptions for the Laplacian.
\cite{STFP-2011} shows that for the directed stochastic blockmodel, the eigenvectors/singular vectors of the adjacency matrix can be orthogonally transformed to approximate the eigenvectors/singular vectors of the population adjacency matrix. \cite{fishkind2012consistent} extends these results to the case when the number of blocks in the stochastic blockmodel are unknown. \cite{Marchette2011VN} also uses techniques closely related to those presented here to investigate the semi-supervised vertex nomination task. 

Finally, another line of work is exemplified by \cite{oliveira2009concentration}. This work shows that, under the independent edge assumption, the adjacency matrix and the normalized Laplacian concentrate around the respective population matrices in the sense of the induced $L^2$ norm. %$\|\cdot\|_2$.
This work uses techniques from random matrix theory. Other work, such as \cite{chung2004spectra}, investigates the spectra of the adjacency and Laplacian matrices for random graphs under a different type of random graph model. 

\section{Framework} \label{sec:frame}
Let $\mathcal{M}_n(A)$ and $\mathcal{M}_{nm}(A)$ denote the set of  $n\times n$ and $n\times m$ matrices with values in $A$ for some set $A$.  Additionally, for $\mathbf{M}\in \mathcal{M}_n(\Re)$, let $\lambda_i(\mathbf{M})$ denote the eigenvalue of $\mathbf{M}$ with the $i^\text{th}$ largest magnitude. All vectors are column vectors.

Let $\mathcal{X}$ be a subset of the unit ball $\mathcal{B}(0,1)\subset\mathbb{R}^{d}$ such that 
%\begin{equation}
%  \label{eq:dotProductBound}
$\langle x_1, x_2 \rangle \in [0,1],$ for all
  $x_1, x_2 \in \mathcal{X}$
%\end{equation}
where $\langle \cdot,\cdot\rangle$ denotes the standard Euclidean inner product. 
Let $F$ be a probability measure on $\mathcal{X}$ and let $X, X_1,X_2,\dotsc,X_n \stackrel{\mathrm{i.i.d.}}{\sim} F$. 
Define $\mathbf{X}:=[X_1,X_2,\dotsc,X_n]^\top :\Omega\mapsto \mathcal{M}_{n,d}(\mathbb{R})$ and $\mathbf{P}:= \mathbf{X}\mathbf{X}^\top:\Omega \mapsto \mathcal{M}_{n}(\mathbb{R})$. 

We assume that the (second moment) matrix $\mathbb{E}[X_1X_1^\top]\in\mathcal{M}_d(\Re)$ is rank $d$ and has distinct eigenvalues $\{\lambda_i(\mathbb{E}[XX^\top])\}$. In particular, we suppose there exists $\delta>0$ such that
\begin{equation}\label{eq:momEigGap}
2\delta < \min_{i\neq j} |\lambda_i(\mathbb{E}[XX^\top ])-\lambda_j(\mathbb{E}[XX^\top ])|
\quad\text{ and }\quad
2\delta<\lambda_d(\mathbb{E}[XX^\top ]). 
\end{equation}

\begin{remark}
The distinct eigenvalue assumption is not critical to the results that follow but is assumed for ease of presentation. The theorems hold in the general case with minor changes.
\end{remark}
\noindent Additionally, we assume that the dimension $d$ of the latent positions is known. %In the discussion, we offer some potential methods to estimate this dimension.

Let $\mathbf{A}$ be a random symmetric hollow matrix such that the entries $\{\mathbf{A}_{ij}\}_{i < j}$ are independent Bernoulli random variables with $\Pr[\mathbf{A}_{ij}=1]=\mathbf{P}_{ij}$ for all $i,j\in [n]$, $i < j$.
 We will refer to $\mathbf{A}$ as the adjacency matrix that corresponds to a graph with vertex set $\{1,\dotsc,n\}$.
Let $\tilde{\mathbf{U}}_\mathbf{A} \tilde{\mathbf{S}}_\mathbf{A} \tilde{\mathbf{U}}_\mathbf{A}^\top $ be the eigen-decomposition
of $|\mathbf{A}|$ where $|\mathbf{A}| =
(\mathbf{A}\mathbf{A}^\top)^{1/2}$
with $\tilde{\mathbf{S}}_\mathbf{A}$ having positive decreasing diagonal entries. Let $\mathbf{U}_\mathbf{A}\in\mathcal{M}_{n,d}(\mathbb{R})$ be given by the
first $d$ columns of $\tilde{\mathbf{U}}_\mathbf{A}\in \mathcal{M}_{n}(\mathbb{R})$ and let
$\mathbf{S}_\mathbf{A}\in \mathcal{M}_{d}(\mathbb{R})$ be given by the first $d$ rows and
columns of $\tilde{S}_\mathbf{A}\in \mathcal{M}_{n}(\mathbb{R})$. Let $\mathbf{U}_\mathbf{P}$ and $\mathbf{S}_\mathbf{P}$ be
defined similarly.

\section{Estimation of Latent Positions}\label{sec:rdpg}
The key result of this section is the following theorem which shows that, using the eigen-decomposition of $|\mathbf{A}|$, we can accurately estimate the true latent positions up to an orthogonal transformation. 
%We state these main results here as a theorem.
\begin{theorem}\label{thm:main}
With probability greater than $1-\frac{2(d^2+1)}{n^2}$, there exists an orthogonal matrix $\mathbf{W} \in
  \mathcal{M}_{d}(\mathbb{R})$ such that
  \begin{equation} 
    \| \mathbf{U}_\mathbf{A} \mathbf{S}_\mathbf{A}^{1/2} \mathbf{W} - \mathbf{X} \| 
    \leq
2d\sqrt{\frac{3\log n}{\delta^3}}.\label{eq:XBnd}
  \end{equation}
Let $\mathbf{W}$ be  as above and define $\hat{\mathbf{X}}=\mathbf{U}_\mathbf{A}\mathbf{S}_\mathbf{A}^{1/2}\mathbf{W}$ with row $i$ denoted by $\hat{X}_i$. Then, for each $i\in[n]$ and all $\gamma<1$,
\begin{equation} 
\Pr[\|\hat{X}_i -X_i\|^2>n^{-\gamma}] = O(n^{\gamma-1} \log n). \label{eq:xiBnd}
\end{equation}
\end{theorem}

We now proceed to prove this result. First, the following result, proved in \cite{STFP-2011}, provides a useful Frobenius bound for the difference between $\mathbf{A}^2$ and $\mathbf{P}^2$.
%%%%%%%%%%%%%%%%%%%%%%%%%%%%%%%%%%%%%%%%%%%%%%%%%%
%%%%%% Proposition bounding \|A^2-P^2\|_F  %%%%%%%
%%%%%%%%%%%%%%%%%%%%%%%%%%%%%%%%%%%%%%%%%%%%%%%%%%
\begin{proposition}[\cite{STFP-2011}]
\label{prop:froBound}
For $\mathbf{A}$ and $\mathbf{P}$ as above, it holds with probability greater than $1-\frac{2}{n^2}$ that
\begin{equation}\label{eq:APfroBound}
 \| \mathbf{A}^2-\mathbf{P}^2\|_F \leq \sqrt{3n^3\log n}. %] \geq 1-\frac{2}{n^2}
\end{equation}
\end{proposition}
The proof of this theorem is omitted and uses the same Hoeffding bound as is used to prove Eq.~\eqref{eq:eigBound1} below.

%%%%%%%%%%%%%%%%%%%%%%%%%%%%%%%%%%%%%%%%%%%%%%%%%%
%%%%%%%%% Proposition providing the eigengaps for P %%%%
%%%%%%%%%%%%%%%%%%%%%%%%%%%%%%%%%%%%%%%%%%%%%%%%%%
\begin{proposition}\label{prop:PeigBound}
For $i\leq d$, it holds with probability greater than $1-\frac{2d^2}{n^2}$ that
\begin{equation}
|\lambda_i(\mathbf{P})-n\lambda_i(\mathbb{E}[X X^\top ])| \leq 2d^2 \sqrt{n\log n},
\label{eq:eigboundProp}
\end{equation}
and for $i>d$, $\lambda_i(P)=0$. If Eq.~\eqref{eq:eigboundProp} holds, then for $i,j\leq d+1$, $i\neq j$ and $\delta$ satisfying Eq.~\eqref{eq:momEigGap} and $n$ sufficiently large, we have 
\begin{equation}  \label{eq:PeigGap}
 |\lambda_i(\mathbf{P}) - \lambda_j(\mathbf{P}) | > \delta n.
\end{equation}
\end{proposition}
\begin{proof}
First, $\lambda_i(P)=\lambda_i(\mathbf{X}\mathbf{X}^\top )=\lambda_i(\mathbf{X}^\top  \mathbf{X})$ for $i\leq d$. Note each entry of $\mathbf{X}^\top \mathbf{X}$ is the sum of $n$ independent random variables each in $[-1,1]$: $\mathbf{X}^\top\mathbf{X} _{ij} = \sum_{l=1}^n X_{li}X_{lj}$. This means we can apply Hoeffding's inequality to each entry of $\mathbf{X}^\top \mathbf{X}-n\mathbb{E}[XX^\top ]$ to obtain
\begin{equation}
 \Pr[|(\mathbf{X}^\top \mathbf{X}-n\mathbb{E}[XX^\top ])_{ij}| \geq 2\sqrt{n\log n}] \leq \frac{2}{n^2}.
\label{eq:eigBound1}
\end{equation}
Using a union bound we have that $\Pr[\|\mathbf{X}^\top\mathbf{X}-\mathbb{E}[XX^\top ]\|_F\geq 2d^2\sqrt{n\log n}] \leq \frac{2d^2}{n^2}$. Using Weyl's inequality \citep{horn85:_matrix_analy}, we have the result. 

Eq.~\eqref{eq:PeigGap} follows from Eq.~\eqref{eq:eigboundProp} provided $2d^2\sqrt{n\log n}<n\delta$, which is the case for $n$ large enough. 
\end{proof}

This next lemma shows that we can bound the difference between the eigenvectors of $\mathbf{A}$ and $\mathbf{P}$, while our main results are for scaled versions of the eigenvectors.
\begin{lemma}\label{lem:eigvecBnd}
With probability greater than $1-\frac{2{d^2+1}}{n^2}$, there exists a choice for the signs of the columns of $\mathbf{U}_\mathbf{A}$ such that for each $i\leq d$,
\begin{equation}
 \| (\mathbf{U}_\mathbf{A})_{\cdot i}- (\mathbf{U}_\mathbf{P})_{\cdot i} \|_F \leq \sqrt{\frac{3\log n}{\delta^2 n}}.
\label{eq:eigVecBound}
\end{equation}
\end{lemma}
\begin{proof}
This is a result of applying the Davis-Kahan Theorem (\cite{davis70}; see also \cite{rohe2011spectral})  to $\mathbf{A}$ and $\mathbf{P}$. Proposition~\ref{prop:froBound} and \ref{prop:PeigBound} give that the eigenvalue gap for $\mathbf{P}^2$ is greater than $\delta^2 n^2$ and that $\|\mathbf{A}^2-\mathbf{P}^2\|_F\leq \sqrt{3n^3 \log n }$ with probabilty greater then $1-\frac{2(d^2+1)}{n^2}$. Apply the
Davis-Kahan theorem to each eigenvector of $\mathbf{A}$ and $\mathbf{P}$, which are the same as the eigenvectors of $\mathbf{A}^2$ and $\mathbf{P}^2$, respectively, to get 
\begin{equation}
\min_{r_i\in\{-1,1\}} \| (\mathbf{U}_\mathbf{A})_{\cdot i}- (\mathbf{U}_\mathbf{P})_{\cdot i} r_i\|_F \leq \sqrt{\frac{3\log n}{\delta^2 n}}
\label{eq:eigveciBound}
\end{equation}
for each $i\leq d$. The claim then follows by choosing $U_A$ so that $r_i=1$ minimizes Eq.~\eqref{eq:eigveciBound} for each $i\leq d$.
\end{proof}

%We have arrived at the following result, which we will then use to prove our main theorem..
We now have the ingredients to prove our main theorem.
%%%%%% MAIN LATENT POSITION ESTIMATION THEOREM %%%%%
%\begin{theorem} With probability greater than $1-\frac{2{d^2+1}}{n^2}$, 
%\begin{equation}
%\|\mathbf{U}_\mathbf{A} \mathbf{S}_\mathbf{A}^{1/2}-\mathbf{U}_\mathbf{P} \mathbf{S}_\mathbf{P}^{1/2}\|_F 
%\leq  2d\sqrt{\frac{3\log n}{\delta^3}}.
%\label{eq:mainBound}
%\end{equation}
%\label{thm:mainBound}
%\end{theorem}
%%%%%%%%%%%%%%%%%%%%%%%%%%PROOF%%%%%%%%%%%%%%%%%%%%%%%%%%%%%%%%%%%
%%%%%%%%%%%%%%%%%%%%%%%%%%PROOF%%%%%%%%%%%%%%%%%%%%%%%%%%%%%%%%%%%
%%%%%%%%%%%%%%%%%%%%%%%%%%PROOF%%%%%%%%%%%%%%%%%%%%%%%%%%%%%%%%%%%
\begin{proof}[Proof of Theorem~\ref{thm:main}]
The following argument assumes that Eqs.~\eqref{eq:eigVecBound} and \eqref{eq:PeigGap} hold, which occurs with probability greater than $1-\frac{2(d^2+1)}{n^2}$. By the triangle inequality, we have
\begin{equation}
\begin{split}
\|\mathbf{U}_\mathbf{A} \mathbf{S}_\mathbf{A}^{1/2}-\mathbf{U}_\mathbf{P} \mathbf{S}_\mathbf{P}^{1/2}\|_F &\leq \|\mathbf{U}_\mathbf{A} \mathbf{S}_\mathbf{A}^{1/2}-\mathbf{U}_\mathbf{A} \mathbf{S}_\mathbf{P}^{1/2}\|_F+\|\mathbf{U}_\mathbf{A} \mathbf{S}_\mathbf{P}^{1/2}-\mathbf{U}_\mathbf{P} \mathbf{S}_\mathbf{P}^{1/2}\|_F\\
& = \|\mathbf{U}_\mathbf{A}(\textbf{S}_\mathbf{A}^{1/2}-\mathbf{S}_\mathbf{P}^{1/2})\|_F + \|(\mathbf{U}_\mathbf{A}-\mathbf{U}_\mathbf{P})\mathbf{S}_\mathbf{P}^{1/2}\|_F.
\end{split}
\label{eq:triIneqBound}
\end{equation}
Note that 
\begin{equation}
\lambda_i^{1/2}(|\textbf{A}|)-\lambda_i^{1/2}(\mathbf{P}) =\frac{\lambda_i^2(|\mathbf{A}|) - \lambda_i^2(\mathbf{P})}{ (\lambda_i(|\mathbf{A}|)+\lambda_i(\mathbf{P}))(\lambda_i(|\mathbf{A}|)^{1/2}+\lambda_i(\mathbf{P})^{1/2})}
\label{eq:factoring}
\end{equation}
where the numerator of the right hand side is less than 
$\sqrt{3 n^3 \log n}$ by Proposition~\ref{prop:froBound}
and the denominator is greater than $(\delta n)^{3/2}$ by
Proposition~\ref{prop:PeigBound}. The first term in
Eq.~\eqref{eq:triIneqBound} is thus bounded by $d\sqrt{3 \log n/\delta^3}$.
For the second term, $(\mathbf{S}_\mathbf{P})_{ii}~\leq~n$ and
$\|\mathbf{U}_\mathbf{A}-\mathbf{U}_\mathbf{P}\|\leq d\sqrt{\frac{3\log n}{\delta^2 n}}$.
%\end{proof}
%
%We now prove our main result, Theorem~\ref{thm:main}.
%The following is an immediate corollary of Theorem~\ref{thm:mainBound}.
%\begin{corollary}
%  \label{cor:1}
%With probability greater than $1-\frac{2(d^2+1)}{n^2}$, there exists an orthgonal matrix $\mathbf{W} \in
%  \mathcal{M}_{d}(\mathbb{R})$ such that
%  \begin{equation}
%    \label{eq:5}
%    \| \mathbf{U}_\mathbf{A} \mathbf{S}_\mathbf{A}^{1/2} \mathbf{W} - \mathbf{X} \| 
%    \leq
%(d+1)\sqrt{\frac{3\log n}{\delta^3}}.
%  \end{equation}
%\end{corollary}
%\begin{proof}[Proof of Theorem~\ref{thm:main}]
%First, we will show that the event in  Eq.~\eqref{eq:XBnd} occurs with probability greater than $1-\frac{2(d^2+1)}{n^2}$.
We have established that  with probability greater than $1-\frac{2{d^2+1}}{n^2}$, 
\begin{equation}
\|\mathbf{U}_\mathbf{A} \mathbf{S}_\mathbf{A}^{1/2}-\mathbf{U}_\mathbf{P} \mathbf{S}_\mathbf{P}^{1/2}\|_F 
\leq  2d\sqrt{\frac{3\log n}{\delta^3}}.
\label{eq:mainBound}
\end{equation}

We now will show that an orthogonal transformation will give us the same bound in terms of $\mathbf{X}$.
  Let $\mathbf{Y} = \mathbf{U}_\mathbf{P} \mathbf{S}_\mathbf{P}^{1/2}$. Then $\mathbf{Y} \mathbf{Y}^\top =
  \mathbf{P} = \mathbf{X} \mathbf{X}^\top$ and thus  
%  \begin{equation}
%    \label{eq:7}
  $  \mathbf{Y} \mathbf{Y}^\top \mathbf{X} = \mathbf{X} \mathbf{X}^\top
    \mathbf{X}$.
%  \end{equation}
 Because
 $\mathrm{rank}(\mathbf{P)} = d = \mathrm{rank}(\mathbf{X})$, we have
 that $\mathbf{X}^\top \mathbf{X}$ is non-singular and hence
% \begin{equation}
%   \label{eq:8}
   $\mathbf{X} = \mathbf{Y} \mathbf{Y}^\top \mathbf{X} (\mathbf{X}^\top
   \mathbf{X})^{-1}$.
% \end{equation}
 Let $\mathbf{W} =  \mathbf{Y}^{\top} \mathbf{X} (\mathbf{X}^\top
   \mathbf{X})^{-1}$. It is straightforward to verify that
   $\mathrm{rank}(\mathbf{W}) = d$ and that $\mathbf{W}^\top\mathbf{ W} = \mathbf{I}$. $\mathbf{W}$ is thus
   an orthogonal matrix, and $\mathbf{X} = \mathbf{Y} \mathbf{W} =  \mathbf{U}_\mathbf{P}
   \mathbf{S}_\mathbf{P}^{1/2} \mathbf{W}$. Eq.~\eqref{eq:XBnd} is thus established.

Now, we will prove Eq.~\eqref{eq:xiBnd}.
Note that because the $\{X_i\}$ are i.i.d., the $\{\hat{X}_i\}$ are exchangeable and hence identically distributed. As a result, each of the random variables $\|\hat{X}_i -X_i\|$ are identically distributed. Note that for sufficiently large $n$, by conditioning on the event in Eq.~\eqref{eq:XBnd}, we have
\begin{equation}
\mathbb{E}[\|\mathbf{X}-\hat{\mathbf{X}}\|^2] \leq \left(1-\frac{2(d^2+1)}{n^2}\right)(2d)^2\frac{3\log n}{\delta^3}+ \frac{2(d^2+1)}{n^2} 2n= O\left(\frac{d^2\log n}{\delta^3}\right)
\end{equation}
because the worst case bound is $\|\mathbf{X}-\hat{\mathbf{X}}\|^2\leq 2n$ with probability 1.
We also have that
\begin{equation}
\mathbb{E}\left[\sum_{i=1}^{n} \mathbb{I}\{\|\hat{X}_i -X_i\|^2>n^{-\gamma}\} n^{-\gamma}\right] \leq \mathbb{E}[\|\mathbf{X}-\hat{\mathbf{X}}\|^2],
\end{equation}
and because the $\|\hat{X}_i -X_i\|$ are identically distributed, the left hand side is simply $n^{1-\gamma}\Pr[\|\hat{X}_i -X_i\|^2>n^{-\gamma}]$.
\end{proof}

\section{Consistent Vertex Classification}\label{sec:knn}

So far we have shown that using the eigen-decomposition of $|\mathbf{A}|$, we can consistently estimate all latent positions simultaneously (up to an orthogonal transformation). 
One could imagine that this will lead to accurate inference for various exploitation tasks of interest.
For example, \cite{STFP-2011} explored the use of this embedding for unsupervised clustering of vertices in the simpler stochastic blockmodel setting.
In this section, we will explore the implications of consistent latent position estimation in the supervised classification setting. 
In particular, we will  prove that universally consistent classification using $k$-nearest-neighbors remains valid when we select the neighbors using the estimated vectors rather than the true but unknown latent positions. 

%\subsection{\texorpdfstring{$k$}{k}-nearest-neighbors}

First, let us expand our framework. 
Let $\mathcal{X}\subset\Re^d$ be as in section~\ref{sec:frame} and let $F_{X,Y}$ be a distribution on $\mathcal{X}\times\{0,1\}$. 
Let $(X_1,Y_1), (X_2,Y_2),\dotsc,(X_n,Y_n),(X_{n+1},Y_{n+1})\stackrel{iid}{\sim} F_{X,Y}$ 
and let $\mathbf{P}\in\mathcal{M}_{n+1}([0,1])$ and $\mathbf{A}\in\mathcal{M}_{n+1}(\{0,1\})$ be as in section~\ref{sec:frame}. Here the $Y_i$s are the class labels for the vertices in the graph corresponding to the adjacency matrix $\mathbf{A}$.
%We also define $\eta(x)=\Pr[Y=1|X=x]$, the posterior probability for class label 1 given $X=x$.

We suppose that we observe only $\mathbf{A}$, the adjacency matrix, and $Y_1,\dotsc,Y_n$, the class labels for all but the last vertex. 
Our goal is to accurately classify this last vertex, so for notational convenience define $X:=X_{n+1}$ and $Y:=Y_{n+1}$. 
%To classify the last vertex, we propose first embedding the adjacency matrix and 
Let the rows of $\mathbf{U}_\mathbf{A}\mathbf{S}_\mathbf{A}^{1/2}$ be denoted by $\zeta_1^\top,\dotsc,\zeta_{n+1}^\top$.
The $k$-nearest-neigbor rule for $k$ odd is defined as follows. 
%Order the the vertices so that the distance from the $\{\zeta_i\}_{i=1}^n$ to $\zeta$ are increasing, with ties being broken so that the vertex with smaller index comes first. 
For $1\leq i \leq n$, let $W_{ni}(X)=1/k$ only if $\zeta_i$ is one of the $k$ nearest points to $\zeta$  from among $\{\zeta_i\}_{i=1}^n$; $W_{ni}(X)=0$ otherwise. (We break ties by selecting the neighbor with the smallest index.)

The $k$-nearest-neighbor rule is then given by $h_n(x)=\mathbb{I}\{\sum_{i=1}^n W_{ni}(X) Y_i > \frac{1}{2}\}$. It is a well known theorem of \cite{stone1977consistent} that, had we observed the original $\{X_i\}$, the $k$-nearest neighbor rule using the Euclidean distance from $\{X_i\}$ to $X$ is universally consistent provided $k\to\infty$ and $k/n\to 0$. This means that for any distribution $F_{X,Y}$, 
\begin{equation}
\mathbb{E}[L_n] := \mathbb{E}[\Pr[\tilde{h}_n(X) \neq Y|(X_1,Y_1), (X_2,Y_2),\dotsc,(X_n,Y_n)]] \to \Pr[h^*(X)\neq Y]=:L^*
\end{equation}
as $n\to \infty$, where $\tilde{h}_n$ is the standard $k$-nearest-neighbor rule trained on the $\{(X_i, Y_i)\}$ and $h^*$ is the (optimal) Bayes rule. This theorem relies on the following very general result, also of \cite{stone1977consistent}, see also \cite{devroye1996probabilistic}, Theorem~6.3. 

\begin{theorem}[\cite{stone1977consistent}]\label{thm:stone}
Assume that for any distribution of $X$, the weights $W_{ni}$ satisfy the following three conditions:
\begin{enumerate}[(i)]
\item \label{stone1} There exists a constant $c$ such that for every nonnegative measurable function $f$ satisfying $\mathbb{E}[f(X)]<\infty$, 
\begin{equation}\label{eq:stone1}
\mathbb{E}\left[ \sum_{i=1}^{n} W_{ni}(X) f(X_i)\right]\leq c \mathbb{E}[f(X)].
\end{equation}
\item \label{stone2} For all $a>0$, 
\begin{equation}\label{eq:stone2}
\lim_{n\to\infty} \mathbb{E}\left[ \sum_{i=1}^{n} W_{ni}(X)\mathbb{I}\{\|X_i-X\|>a\} \right]=0
\end{equation}
\item \label{stone3}
\begin{equation}
\lim_{n\to\infty} \mathbb{E}\left[\max_{1\leq i\leq n} W_{ni}(X) \right]=0
\end{equation}
\end{enumerate}
Then $h_n(x) = \mathbb{I}\{\sum_i W_{ni}(x) > 1/2\}$ is universally consistent.
\end{theorem}

\begin{remark}
Recall that the $\{\hat{X_i}\}$ are defined in Theorem~\ref{thm:main}. Because the $\{\hat{X}_i\}$ are obtained via an orthogonal transformation of the $\{\zeta_i\}$, the nearest neighbors of $\hat{X}=\hat{X}_{n+1}$ are the same as those of $\zeta$.  As a result of this and the relationship between $\mathbf{X}$ and $\hat{\mathbf{X}}$, we work using the $\{\hat{X}_i\}$, even though these cannot be known without some additional knowledge.
\end{remark}

To prove that the $k$-nearest-neighbor rule for the $\{\hat{X}_i\}$ is universally consistent, we must show that the corresponding $W_{ni}$ satisfy these conditions. The methods to do this are adapted from the proof presented in \cite{devroye1996probabilistic}. We will outline the steps of the proof, but the details follow {\em mutatis mutandis} from the standard proof.

First, the following Lemma is adapted from \cite{devroye1996probabilistic} by using a triangle inequality argument.
\begin{lemma}%\label{lem:closenn}
Suppose $k/n\to0$. If $X\in \mathrm{supp}(F_X)$, then $\|\hat{X}_{(k)}(\hat{X})-\hat{X}\|\to 0$ almost surely, where $\hat{X}_{(k)}(\hat{X})$ is the $k$-th nearest neighbor of $\hat{X}$ among $\{\hat{X}_i\}_{i=1}^n$.
\end{lemma}

Condition~\eqref{stone3} follows immediately from the definition of the $W_{ni}$. The remainder of the proof follows with few changes after recognizing that the random variables $\{(X,\hat{X})\}$ are exchangeable. 
Overall, we have the following universal consistency result.
\begin{theorem} \label{thm:univCons}
If $k\to\infty$ and $k/n\to 0$ as $n\to\infty$, then the $W_{ni}(X)$ satisfy the condtions of Theorem~\ref{thm:stone} and hence $\mathbb{E}[\Pr[h_n(\hat{X})\neq Y| \mathbf{A},\{Y_i\}_{i=1}^n]=\mathbb{E}[L_n]\to L^*_X$.
\end{theorem}

\section{Extensions} \label{sec:disc}
The results presented thus far are for the specific problem of determining one unobserved class label for a vertex in a random dot product graph. In fact, the techniques used can be extended to somewhat more general settings without significant additional work.

\subsection{Classification}
For example, the results in section~\ref{sec:knn} are stated in the case that we have observed the class labels for all but one vertex. However, the universal consistency of the $k$-nearest-neighbor classifier remains valid provided the number of vertices $m$ with observed vertex class labels goes to infinity and $k/m\to 0$ as the number of vertices $n\to\infty$. In other words, we may train the $k$-nearest neighbor  on a smaller subset of the estimated latent vectors provided the size of that subset goes to $\infty$. 

On the other hand, if we fix the number of observed class labels $m$ and the classification rule $h_m$ and let the number of vertices tend to $\infty$, then we can show the probability of incorrectly classifying a vertex will converge to $L_m=\Pr[h_m(Z)\neq Y]$. Additionally, our results also hold when the class labels $Y$ can take more than two but still finitely many values.

In fact, the results in section~\ref{sec:knn} and Eq.~\eqref{eq:xiBnd} from Theorem~\ref{thm:main} rely only on the fact that the $\{X_i\}$ are i.i.d.\ and bounded, the $\{(X_i,\hat{X}_i)\}$ are exchangeable, and $\|\mathbf{X}-\hat{\mathbf{X}}\|_F^2$ can be bounded with high probability by a $O(\log n)$ function. The random graph structure provided in our framework is of interest, but it is the total noise bounds that are crucial for the universal consistency claim to hold. 

\subsection{Latent Position Estimation}
In section~\ref{sec:rdpg}, we state our results for the random dot product graph model. We can generalize our results immediately by replacing the dot product with a bi-linear form, $g(x,y)=x^\top (\mathbf{I}_{d'} \oplus(-\mathbf{I}_{d''}))y$, where $\mathbf{I}_d$ is the $d\times d$ identity matrix. 
This model has the interpretation that similarities in the first $d'$ dimensions increase the probability of adjacency, while similarities in the last the last $d''$ reduce the probability of adjacency. All the results remain valid under this model, and in fact, arguments in \cite{oliveira2009concentration} can be used to show that the signature of the bi-linear form can also be estimated consistently. 
We also recall that the assumption of distinct eigenvalues for $\mathbb{E}[XX^T]$ can be removed with minor changes. Particularly, Lemma~\ref{lem:eigvecBnd} applies to groups of eigenvalues, and subsequent results can be adapted without changing the order of the bounds.

This work focuses on undirected graphs and this assumption is used explicitly throughout section~\ref{sec:rdpg}.
We believe moderate modifications would lead to similar results for directed graphs, such as in \cite{STFP-2011}; however at present we do not investigate this problem.
%We also consider that the edges are not weighted, but our results extend easily to weighted edges. The results hold provided the edge weights random variables that are uniformly bounded almosu and have expectation $\mathbf{P}_{ij}$.
We also note that we assume the graph has no loops so that $\mathbf{A}$ is hollow. 
This assumption can be dropped, and in fact, the impact of the diagonal is asymptotically negligible, provided each entry is bounded. 
\cite{Marchette2011VN} suggest that augmenting the diagonal may improve latent position estimation for finite samples.

In \cite{rohe2011spectral}, the number of blocks in the stochastic blockmodel, which is related to $d$ in our setting \citep{STFP-2011}, is allowed to grow with $n$; our work can also be extended to this setting. In this case, it will be the interaction between the rate of growth of $d$ and the rate that $\delta$ vanishes that controls the bounds in Theorem~\ref{thm:main}. Additionally, the consistency of $k$-nearest-neighbors when the dimension grows is less well understood and results such as Stone's Theorem~\ref{thm:stone}  do not apply.

In addition to keeping $d$ fixed, we also assume that $d$ is known. \cite{fishkind2012consistent} and \cite{STFP-2011} suggest consistent methods to estimate the latent space dimension. The results in \cite{oliveira2009concentration} can also be used to derive thresholds for eigenvalues to estimate $d$. 

Finally, \cite{fishkind2012consistent} and \cite{Marchette2011VN} also consider that the edges may be attributed; for example, if edges represent a communication, then the attributes could represent the topic of the communication. The attributed case can be thought of as a set of adjacency matrices, and we can embed  each separately and concatenate the embeddings. \cite{fishkind2012consistent} argues that this method works under the attributed stochastic blockmodel and similar arguments could likely be used to extend the current work. 

\subsection{Extension to the Laplacian}
The eigen-decomposition of the graph Laplacian is also widely used for similar inference tasks. In this section, we argue informally that our results extend to the Laplacian. We will consider a slight modification of the standard normalized Laplacian as defined in \cite{rohe2011spectral}. This modification scales the Laplacian in \cite{rohe2011spectral} by $n-1$ so that the first $d$ eigenvalues of our matrix are $O(n)$ rather then $O(1)$ for the standard normalized Laplacian.

Let $\mathbf{L}:=\mathbf{D}^{-1/2}\mathbf{A}\mathbf{D}^{-1/2}$ where $\mathbf{D}$ is diagonal with $\mathbf{D}_{ii} := \frac{1}{n-1}\sum_{i=1}^n \mathbf{A}_{ij}$. Additionally, let $\mathbf{Q}:=\bar{\mathbf{D}}^{-1/2}\mathbf{P}\bar{\mathbf{D}}^{-1/2}$
where $\bar{\mathbf{D}}$ is diagonal with 
\begin{equation}
\bar{\mathbf{D}}_{ii} := \frac{1}{n-1}\mathbb{E}\left[\sum_{j=1}^n \mathbf{A}_{ij} | \ \mathbf{X}\right] =  \frac{1}{n-1}\sum_{j\neq i} \mathbf{P}_{ij} = \frac{1}{n-1} \sum_{j\neq i} \langle X_i,X_j\rangle.
\end{equation} 
Finally, define $q:\Re^d\times\Re^d\mapsto\Re^d$ as 
$q(x,y):=\frac{x}{\sqrt{\langle x,y\rangle}}$, 
 $Z_i:= q(X_i,\frac{1}{n}\sum_{j\neq i} X_j)$
  and $\tilde{Z}_i:=q(X_i,\mathbb{E}(X))$.
%\begin{equation}
%Z_i = \frac{X_i}{\sqrt{\langleX_i, \frac{1}{n}\sum_{j\neq i} X_j\rangle}} 
%\quad\text{and}\quad
%\tilde{Z}_i =  \frac{X_i}{\sqrt{\langleX_i, \mu\rangle}},
%\text{ where } \mu=\mathbb{E}[X]. 
%\end{equation}

Because the pairwise dot products of the rows of $\bar{\mathbf{D}}^{-1/2}\mathbf{X}$ are the same as the entries of $\mathbf{Q}$,  the scaled eigenvectors of $\mathbf{Q}$ must be an orthogonal transformation of the $\{Z_i\}$. 
Further, note that for large $n$, $Z_i$ and $\tilde{Z_i}$ will be close with high probability because $ \frac{1}{n}\sum_{j\neq i} X_j\stackrel{a.s}{\to}\mathbb{E}[X]$ and the function $q(X_j,\cdot)$ is smooth almost surely.
Additionally, the $\{\tilde{Z}_i\}$ are i.i.d.\ and $q(\cdot,\mathbb{E}[X])$ is one-to-one so that the Bayes optimal error rate is the same for the $\{\tilde{Z}_i\}$ as for the $\{X_i\}$: $L^*_X = L^*_{\tilde{Z}}$.
If the further assumption that the minimum expected degree among all vertices is greater than $\sqrt{2}n/\sqrt{\log n}$ holds, then the assumptions of Theorem~2.2 in \cite{rohe2011spectral} are satisfied. 

Let $\hat{Z}_i$ denote the $i^\mathrm{th}$ row of the matrix  $\mathbf{U}_\mathbf{L}\mathbf{S}_\mathbf{L}$ defined analogously to section~\ref{sec:frame} and let $\tilde{\mathbf{Z}}$ be the matrix with row $i$ given by $\tilde{Z}_i^\top$. %, then we can consider using $k$-nearest-neighbors on this matrix as well.
Using the results in \cite{rohe2011spectral} and similar tools to those we have used thus far, one can show that $\min_{\mathbf{W}}\|\mathbf{U}_\mathbf{L}\mathbf{S}_\mathbf{L}\mathbf{W}-\tilde{\mathbf{Z}}\|^2$ can be bounded with high probability by a function in $O(\log n)$. As discussed above, this is sufficient for $k$-nearest-neighbors trained on $\{(\hat{Z}_i,Y_i)\}$ to be universally consistent.
In this paper we do not investigate the comparative values of the eigen-decompositions for the Laplacian versus the adjacency matrix, but one factor may be the properties of the map $q$ defined above as applied to different distributions on $\mathcal{X}$. 
%  Using this result and arguments similar to those in presented thus far, one can show that the conditions for universal consistency of $k$-nearest-neighbors hold when we use the eigen-decomposion of $\mathbf{L}$. As discussed earlier, we these conditions, will require that 

%This result extends to the eigen-decomposition of the Laplacian as well.
%The following are basically notes
%
% Let $\mathbf{L}:=\mathbf{D}^{-1/2}\mathbf{A}\mathbf{D}^{-1/2}$ where $\mathbf{D}$ is diagonal with $\mathbf{D}_{ii} = \sum_{i=1}^n \mathbf{A}_{ij}$. Additionally, let $\bar{\mathbf{Q}}=\bar{\mathbf{D}}^{-1/2}\mathbf{P}\bar{\mathbf{D}}^{-1/2}$
%
%
%$Z_i = \frac{X_i}{\sqrt{\langleX_i, \mathbb{E}[X]\rangle}}$; 
%$\tilde{Z}_i =  \frac{X_i}{\sqrt{\langleX_i, \bar{X}\rangle}}$

\section{Experiments} \label{sec:emp}
In this section we present empirical results for a graph derived from Wikipedia links as well as simulations for an example wherein the $\{X_i\}$ arise from a Dirichlet distribution. 

\subsection{Simulations} \label{sec:sim}

\begin{figure}[t]
\begin{center}
\includegraphics[width=.7\textwidth]{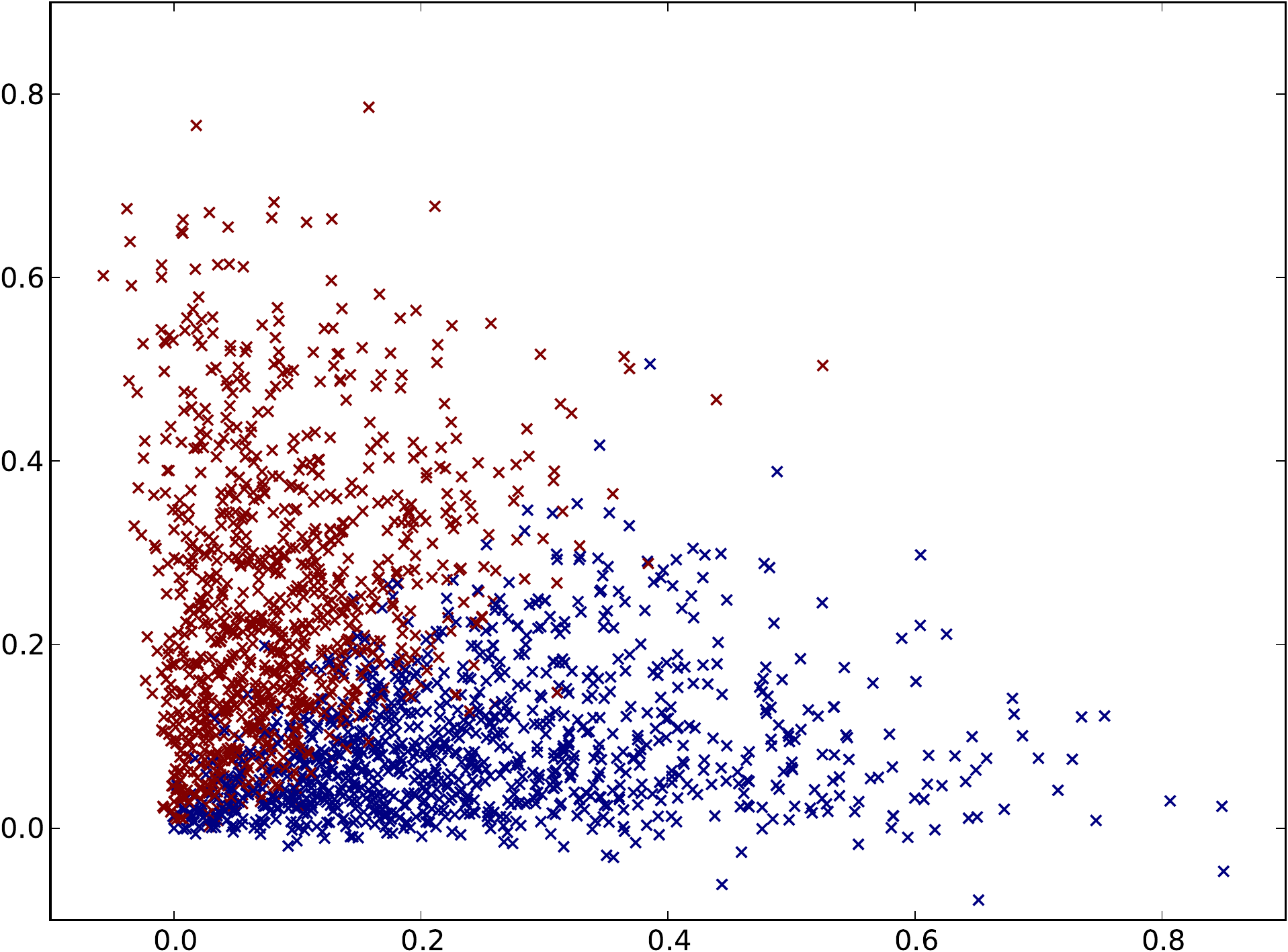}
\end{center}
\caption{An example of estimated latent position $\{\hat{X}_i\}$ for the distribution described in section~\ref{sec:sim}. Each point is colored according to class labels $\{Y_i\}$. For the original latent position $\{X_i\}$, the two classes would be perfectly separated by the line $y=x$. In this figure the two classes are nearly separated but have some overlap. Note also that some estimated positions are outside the support of the original distribution.}
\label{fig:simScatter}
\end{figure}

To demonstrate our results, we considered a problem where perfect classification is possible. Each $X_i:\Omega \mapsto\Re^2$ is distributed according to a Dirichlet distribution with parameter $\alpha=[2,2,2]^\top$ where we keep just the first two coordinates. The class labels are determined by the $X_i$ with $Y_i=\mathbb{I}\{X_{i1}<X_{i2}\}$ so in particular $L^*=0$.

For each $n\in\{100,200,\dots,2000\}$, we simulated 500 instances of the $\{X_i\}$ and sample the associated random graph. For each graph, we used our technique to embed each vertex in two dimensions. \
To facilitate comparisons, we used the matrix $\mathbf{X}$ to construct the matrix $\hat{\mathbf{X}}$ via transformation by the optimal orthogonal $\mathbf{W}$. Figure~\ref{fig:simScatter} illustrates our embedding for $n=2000$ with each point corresponding to a row of $\hat{\mathbf{X}}$ with points colored according the class labels $\{Y_i\}$. To demonstrate our results from section~\ref{sec:rdpg}, figure~\ref{fig:simFroErr} shows the average square error in the latent position estimation per vertex.

\begin{figure}
\begin{center}
\includegraphics[width=.9\textwidth]{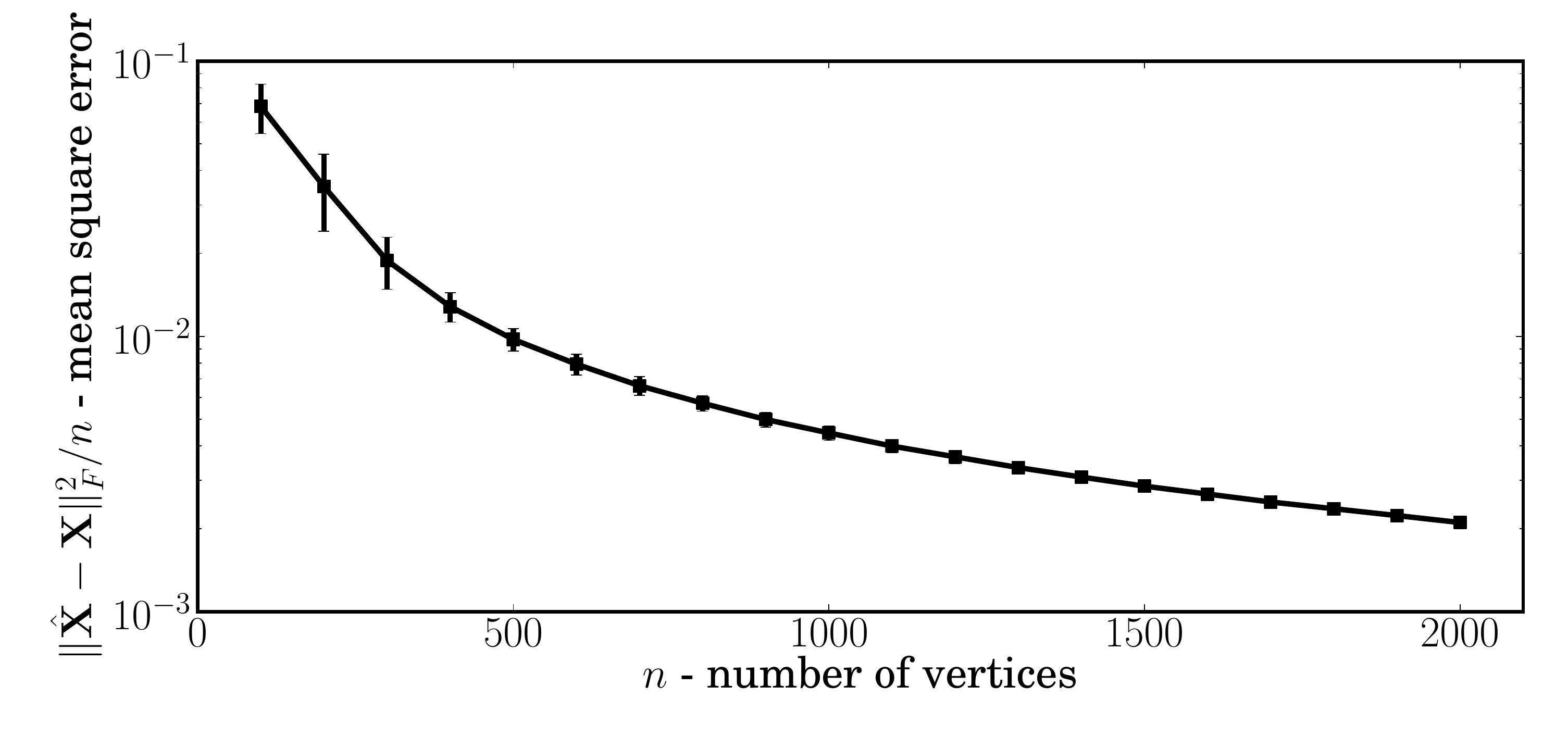}
\end{center}
\caption{Mean square error versus number of vertices. This figure shows the mean square error in latent position estimation per vertex, given by $\|\hat{\mathbf{X}}-\mathbf{X}\|_F^2/n$,  for the simulation described in section~\ref{sec:sim}. 
%On average, the estimated latent positions converge rapidly to the true latent positions as the number of vertices in the graph increases. 
The error bars are given by the standard deviation of the average square error over 500 monte carlo replicates for each $n$. On average, the estimated latent positions converge rapidly to the true latent positions as the number of vertices in the graph increases.}
\label{fig:simFroErr}
\end{figure}

For each graph, we used leave-one-out cross validation to evaluate the error rate for $k$-nearest-neighbors for $k=2\lfloor\sqrt{n}/4\rfloor +1 $. We suppose that we observe all but 1 class label as in section~\ref{sec:knn}. Figure~\ref{fig:simLhatBoxplot} shows the classification error rates. The black line shows the classification error when classifying using $\hat{\mathbf{X}}$ while the red line shows the classification error when classifying using $\mathbf{X}$. Unsurprisingly, classifying using $\hat{\mathbf{X}}$ gives worse performance. However we still see steady improvement as the number of vertices increases, as predicted by our universal consistency result. Indeed, this figure suggests that the rates of convergence may be similar for both $\mathbf{X}$ and $\hat{\mathbf{X}}$.

\begin{figure}
\begin{center}
\includegraphics[width=.9\textwidth]{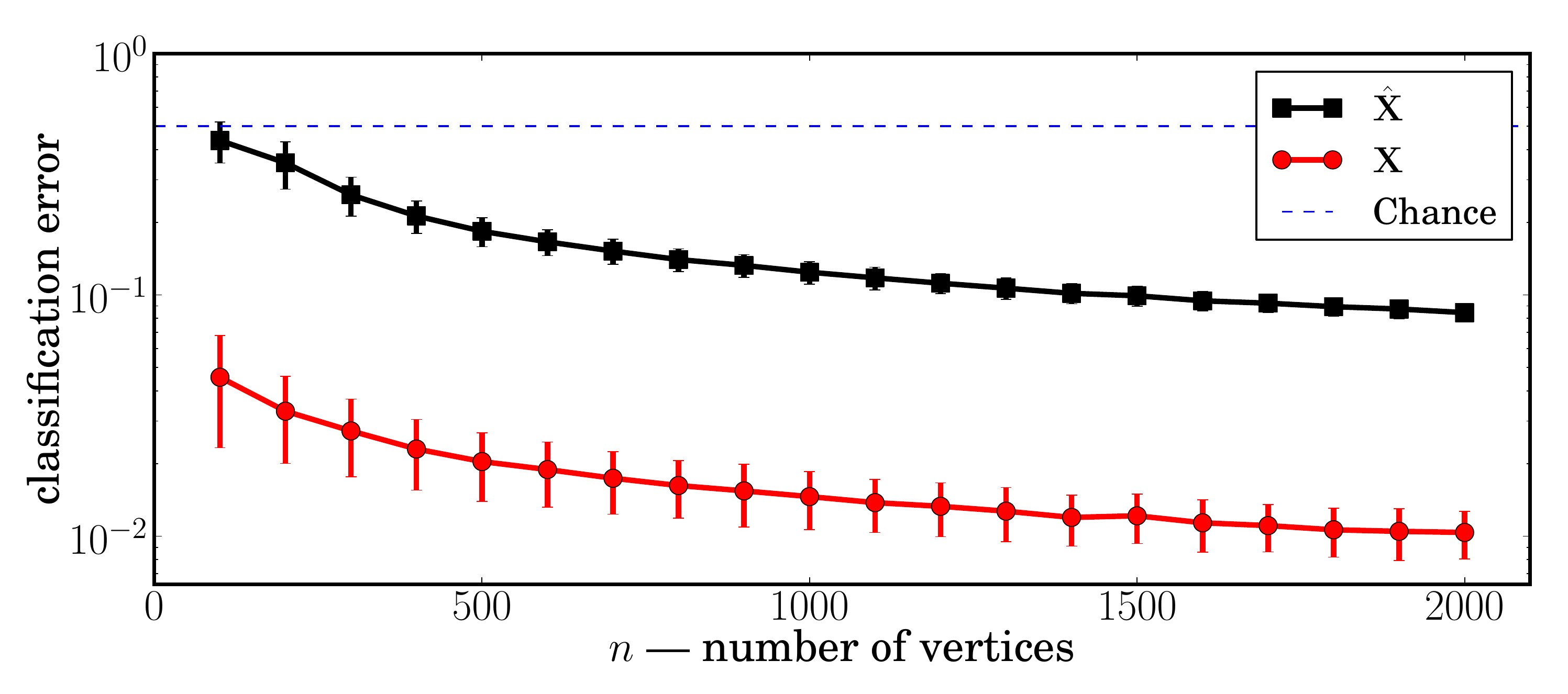}
\end{center}
\caption{Leave-one-out cross validation classification error estimates using $k$-nearest neighbors for the simulations described in section~\ref{sec:sim}. The black line show the classification error when classifying using $\hat{\mathbf{X}}$ while the red line shows the error rates when classifying using $\mathbf{X}$. Error bars show the standard deviation over the 500 monte carlo replicates. Chance classification error is $0.5$; $L^*=0$. This figure suggests the rates of convergence may be similar for both $\mathbf{X}$ and $\hat{\mathbf{X}}$. }
\label{fig:simLhatBoxplot}
\end{figure}

\subsection{Wikipedia Graph}
For this data (\cite{ma2012fusion}, \url{http://www.cis.jhu.edu/~zma/zmisi09.html}), each vertex in the graph corresponds to a Wikipedia page and the edges correspond to the presence of a hyperlink between two pages (in either direction). We consider this as an undirected graph. Every article within two hyperlinks of the article ``Algebraic Geometry'' was included as a vertex in the graph. This resulted in $n=1382$ vertices.  Additionally, each document, and hence each vertex, was manually labeled as one of the following: Category (119), Person (372), Location (270), Date (191) and Math (430).

To investigate the implications of the results presented thus far, we performed a pair of illustrative investigations.
First, we used our technique on random induced subgraphs and used leave-one-out cross validation to estimate error rates for each subgraph. We used $k=9$ and $d=10$ and performed 100 monte carlo iterates of random induced subgraphs with $n\in\{100,200,\dotsc,1300\}$ vertices. Figure~\ref{fig:wiki_subgraph} shows the mean classification error estimates using leave-one-out cross validation on each randomly selected subgraph. Note, the chance error rate is $1-430/1382 =0.689$.
%Comparing figures~\ref{fig:wiki_seesub} and \ref{fig:wiki_subgraph} suggests that there are trade-offs between the number of vertices observed and the number of class labels observed. This trade-off represents a key area of future study, especially when selecting methods for collecting graph data. 

\begin{figure}
\begin{center}
\includegraphics[width=\textwidth]{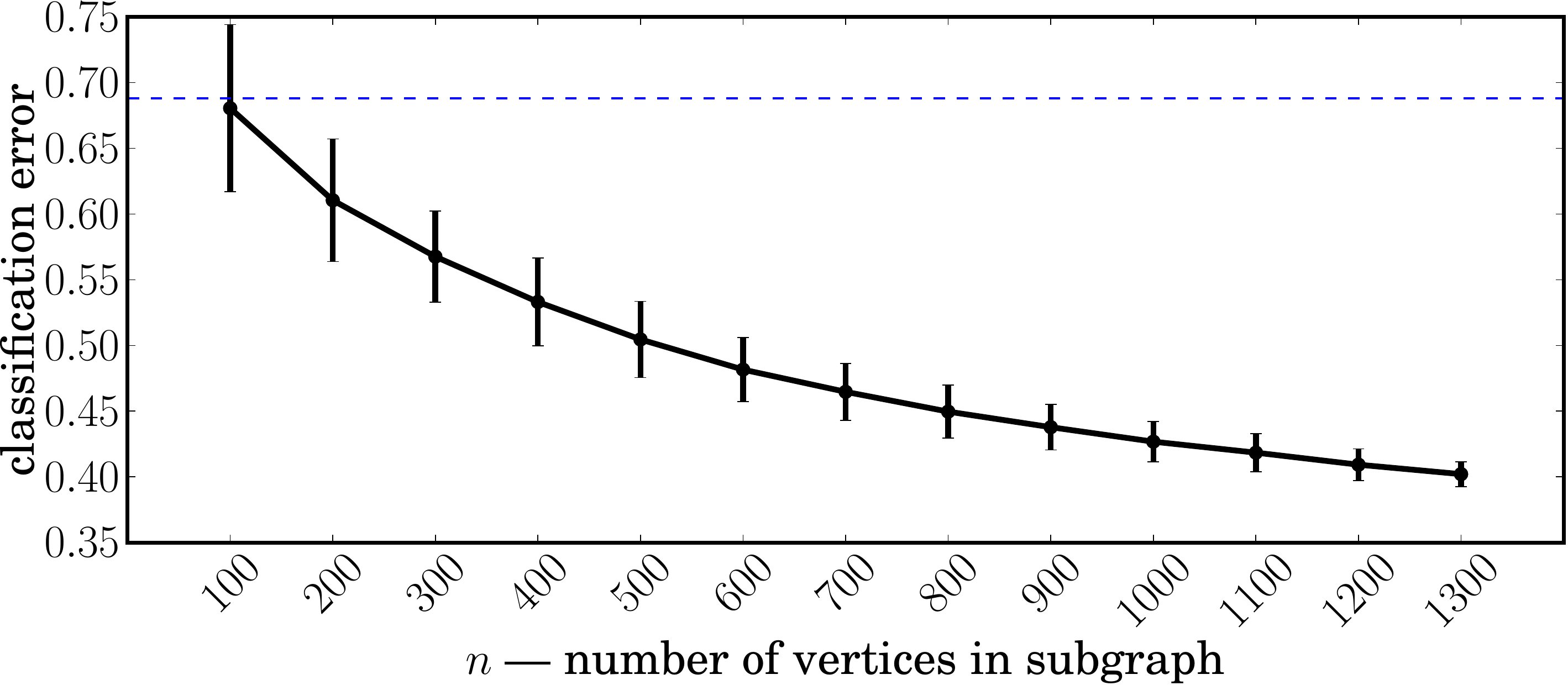}
\end{center}
\caption{Error rate using leave-one-out cross validation for random induced subgraphs. Chance classification error is $\approx 0.688$ shown in blue. This illustrates the improvement vertex classification as the number of vertices and the number of observed class labels increases.	}
\label{fig:wiki_subgraph}
\end{figure}

We also investigated the performance of our procedure for different choices of $d$, the embedding dimension, and $k$, the number of nearest neighbors. Because this data has 5 classes, we use the standard $k$-nearest-neighbor algorithm and break ties by choosing the first label as ordered above.
Using leave-one-out cross validation, we calculated an estimated error rate for each $d\in \{1,\dotsc,50\}$ and $k\in\{1,5,9,13,17\}$. The results are shown in Figure~\ref{fig:wiki_kd}.  This figure suggests that our technique will be robust to different choices of $k$ and $d$ within some range.

\begin{figure}[t!]
\begin{center}
\includegraphics[width=.9\textwidth]{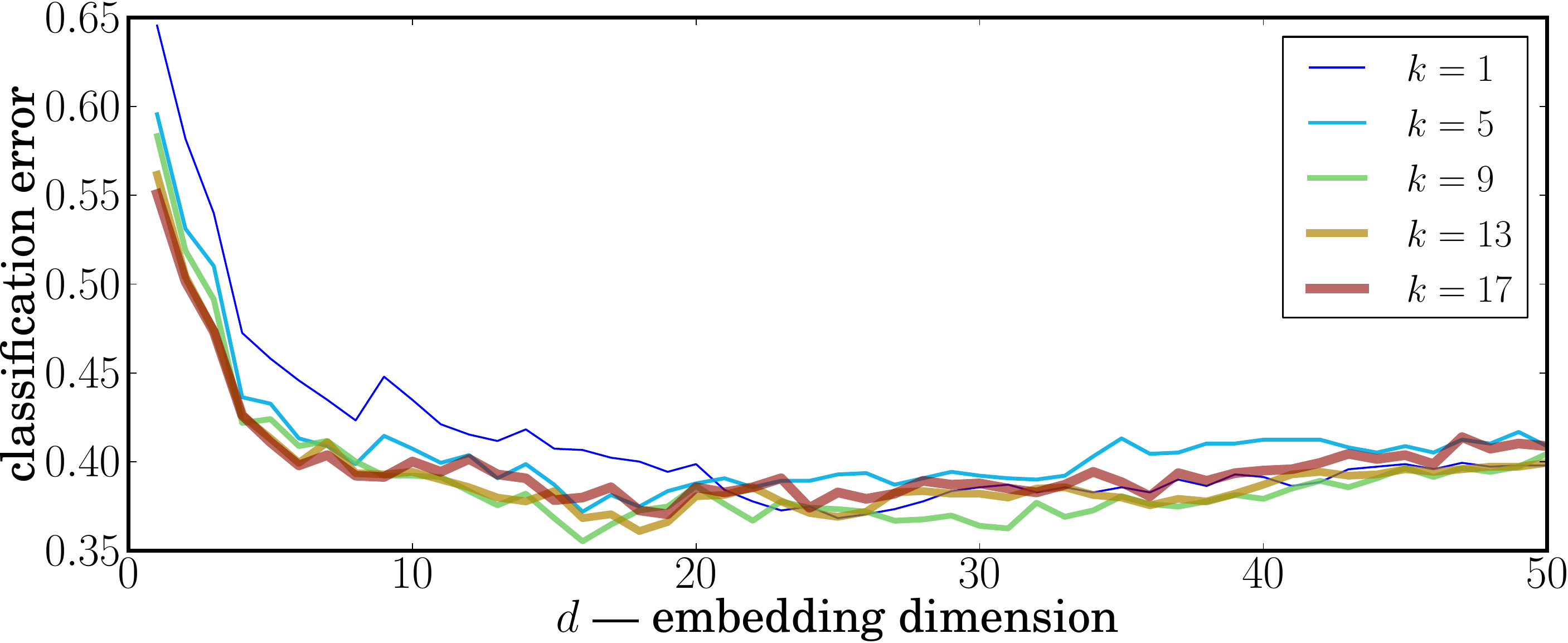}
\end{center}
\caption{Leave-one-out error rate plotted against the embedding dimension $d$ for different choices of $k$ (see legend). Each line corresponds to a different choice for the number of nearest neighbors $k$. All results are better than chance $\approx 0.688$. We see that method is robust to changes of $k$ and $d$ near the optimal range.}
\label{fig:wiki_kd}
\end{figure}
%
%We also compared the error rates as we increase the fraction of the vertices for which we observe class labels. Here, we randomly selected a certain percentage of class labels for training and classified the remaining vertices. We used $k=9$ and $d=10$ and performed 100 monte carlo iterates for observed fractions in $\{0.1, 0.15,\dotsc,0.9\}$. Figure~\ref{fig:wiki_seesub} shows the classification error for each training set size.
%
%\begin{figure}
%\begin{center}
%\includegraphics[width=\textwidth]{Wiki_errorByFractionObserved_errbar}
%\end{center}
%\caption{Error rate according to the fraction of the graph used for training. Error rate was computed by testing on the remaining graph. Chance classification error is $\approx 0.688$ shown in blue. This figure suggests that, to a point, there are increasing benefits to performance as we observe more class labels for the graph.}
%\label{fig:wiki_seesub}
%\end{figure}

\section{Conclusion}
Overall, we have shown that under the random dot product graph model, we can consistently estimate the latent positions provided they are independent and identically distributed. We have shown further that these estimated positions are also sufficient to consistently classify vertices. We have shown that this method works well in simulations and can be useful in practice for classifying documents based on their links to other documents.

\section*{References}
\bibliographystyle{plainnat}
\bibliography{LstarXhat}
\end{document}